\documentclass[twoside,leqno,twocolumn]{article}

\usepackage[letterpaper]{geometry}
\usepackage{amssymb}
\usepackage{amsmath}
\usepackage{ltexpprt}
\usepackage{graphicx}
\usepackage{bm}

\usepackage{float}
\usepackage{algorithm,algorithmic}
\usepackage{hyperref}

\begin{document}

\title{\Large Density of States Graph Kernels}

\author{Leo Huang\thanks{Cornell University, Ithaca, NY, 14853, USA} \\ {\texttt{ah839@cornell.edu}} \and
Andrew J. Graven$^*$\\ {\texttt{andrew@graven.com}} \and
David Bindel$^*$ \\ { \texttt{bindel@cornell.edu}}
}

\date{}

\maketitle


\fancyfoot[R]{\scriptsize{Copyright \textcopyright\ 2021 by SIAM\\
Unauthorized reproduction of this article is prohibited}}





\begin{abstract} \small\baselineskip=9pt 
A fundamental problem on graph-structured data is that of quantifying similarity between graphs. Graph kernels are an established technique for such tasks; in particular, those based on random walks and return probabilities have proven to be effective in wide-ranging applications, from bioinformatics to social networks to computer vision. However, random walk kernels generally suffer from slowness and \textit{tottering}, an effect which causes walks to overemphasize local graph topology, undercutting the importance of global structure. To correct for these issues, we recast return probability graph kernels under the more general framework of \textit{density of states} --- a framework which uses the lens of spectral analysis to uncover graph motifs and properties hidden within the interior of the spectrum --- and use our interpretation to construct scalable, composite density of states based graph kernels which balance local and global information, leading to higher classification accuracies on benchmark datasets.
\end{abstract}
\section{Introduction}
\subsection{Classifying Large Graphs}
Large-scale graph classification tasks abound in a computer vision, social science, and other fields. Communities within social networks can be represented as graphs, with individuals as nodes and connections between them as edges. In the Facebook and Twitter world graphs, communities easily number in the tens of thousands on average, and pages attract millions of followers \cite{nbcqanon}. Large-scale graph classification algorithms can be used to tell apart communities with different characteristics, for example, and therefore has implications for maintaining community integrity and internet safety. Collaboration networks of researchers can have hundreds of thousands of nodes and millions of edges \cite{collabsize} --- a question of interest is whether such sub-networks can be classified according to its research subfield solely based on network structural properties. Fine-grained and complex grid networks arise in texture classification, a subfield of computer vision \cite{PropKer}. Indeed, large graphs are prevalent in the real world. The scalability of algorithms on such graphs is indubitably one of the most pressing challenges in data science \cite{realgraphs}.

There are two distinct, yet interrelated lines of work which address the problem of classification at the graph level. They are, in order of conception, \textit{graph kernels} and \textit{graph neural networks} (GNNs) \cite{GNNsurvey}. Graph kernels are kernel functions which are used to compute inner products of graphs by exploiting structure or edge or vertex label information. As such, they enable kernel learning algorithms, such as support vector machines (SVMs) to operate directly on graphs. The literature on graph kernels is expansive: examples of classical graph kernels include the vertex/edge label, random walk \cite{fastrandom}, spectral decomposition, Weisfeiler-Lehman kernel \cite{wl}, and kernels of the optimal assignment (OA) variety \cite{wloa}. While earlier graph kernels were designed for plain unattributed graphs, later ones have been adapted to graphs with continuous and categorical edge and vertex attributes. Recent works have also incorporated Wasserstein/EMD distances to compute similarity between latent space embeddings (\cite{wassleh}, \cite{fewshot}), including \cite{aaai}, which proposes algorithms for computing an EMD and pyramid matching based kernel to compare global properties of graphs. More recently, \cite{gromov} used a Gromov-Wasserstein learning framework for the similar problem of graph matching. For a survey and comparison of graph kernels, we refer to \cite{graphkernels}.  

Graph neural networks (GNNs) use multi-layer architecture and non-linear activation functions to extract high-order features from graphs structured data. GNNs generally have runtime complexity proportional to $|E|$, but are generally expensive to train. Examples include the graph isomorphism network (GIN) \cite{gin}, deep graph kernels \cite{dgk}, $\text{PATCHY-SAN}$, and Graph Neural Tangent Kernels \cite{gntk}. Recent studies have shown that graph kernels and graph neural networks remain competitive with one another \cite{tudata}.
 
Efforts have been made to collectivize graph kernel benchmark data sets and standardize evaluation procedures. Two such databases are \text{TUDatasets} and Open Graph Benchmark (OGB) (\cite{ogb}, \cite{KKMMN2016}, \cite{tudata}). The OGB and TUData projects both aim to assemble graph datasets from wide ranging domains. They provide tools for loading data, forming train-test splits, and facilitating benchmarking in Python and PyTorch.  

\subsection{Density of States}
The \textit{density of states} (DOS) induced by a real symmetric matrix $A \in \mathbb{R}^{N\times N}$ with eigenpairs $(\lambda_i,q_i), i\in \{1,2,\hdots,N\}$ is defined as the generalized function
\label{eqn:dos}
\begin{equation}\mu(\lambda) = \frac{1}{N} \sum \delta(\lambda - \lambda_i)\end{equation}
$\mu$ is also referred to as the spectral density of $A$. For any vector $u\in\mathbb{R}^n$, the local density of states is given by
\begin{equation}
\mu(\lambda, u) = \sum_{i=1}^N |u^Tq_i|^2\delta(\lambda-\lambda_i)
\end{equation}
The point-wise density of states (PDOS) is obtained by setting $u = e_k$. In \cite{NDS}, DOS was applied in analyzing real world networks with up to several billion nodes. Recent developments have enabled the computation of complete eigenspectra, which can be taken as detailed spectral signatures for graphs. The difficulty remains, however, to construct a suitable graph kernel using these spectral signatures and conduct comparisons on the basis of these features at the graph level. 

\subsection{Related Work}

The current state-of-the-art graph kernels based on return probabilities include $\text{RetGK}_{\text{I}}$ \cite{retp} and its adaptation to labeled graphs, \textit{scalable attributed graph embeddings}, $\text{SAGE}$ \cite{sagegk}. Both employ a return probability feature vector (RPF), defined as
\begin{equation}
    \vec{p}_i = [\textbf{P}_G^1(i, i), \textbf{P}_G^2(i, i), ..., \textbf{P}_G^S(i, i)]^T
\end{equation}
where $\textbf{P}_G^s(i, i), s=1, 2, ..., S$ is the probability that an $s$-step random walk on $G$ starting from $v_i$ returns to itself. The RPF can be computed using an eigendecompositional approach, and enjoys several desirable properties, such as isomorphism-invariance, multi-resolution, and informativeness. Graphs with different RPFs are not isomorphic. A similarity in RPFs generally indicates similarity in spectral information and local node structural roles. 

Both $\text{RetGK}_{\text{I}}$ and SAGE scale linearly with the number of graphs in the dataset, and with the size of the graph embedding. However, they do not scale with $|E|$, the number of edges in the graph. Other well-known graph kernels such as WL do scale with $|E|$; however, their discriminatory power does not surpass that of $\text{RetGK}_{\text{I}}$. 

Although $\cite{retp}$ proposes a fast Monte-Carlo simulation approach to computing the RPF, $\text{RetGK}_{\text{MC}}$, this approach has poorer accuracy compared to $\text{RetGK}_{\text{I}}$ (full eigendecomposition) and $\text{RetGK}_{\text{II}}$ (full eigendecomposition with approximate feature maps). On the other hand,
full eigendecomposition has cost $O(N^3)$, and becomes infeasible for graphs with tens of thousands of nodes. 

A graph kernel which uses spectral information is the $\textit{family of graph spectral distances}$ ($\text{FGSD}$) \cite{fgsd}. This paper introduces a graph representation based on the multiset of node pairwise spectral distances. For each graph, they define the $\textit{graph spectrum}$ $\mathcal{R}:=\{\mathcal{S}_f(x, y)|\forall (x, y)\in V\}$, where $S_f(x, y) = \sum_{k=0}^{N-1}f(\lambda_k)(\phi_k(x)-\phi_k(y))^2$, where $\phi_k(x)$ is the $x$-entry value of the eigenvector $\phi_k$ of the Laplacian of $G$ and $f$ is a monotonic function. For classification, they use a discretization $\mathcal{F}$ of $\mathcal{R}$, namely the histogram of $\mathcal{R}$ with binwidth size chosen from $\{1\mathrm{e}{-3}, 1\mathrm{e}{-4}, 1\mathrm{e}{-5}\}$, and also select the family of biharmonic distances for $f$. Importantly, they demonstrate desirable uniqueness properties of the graph spectrum and exhibit a fast approximation algorithm for computing $\mathcal{F}$ based on Chebyshev series expansion of $f$. Similar to our methods, $\text{FGSD}$ is a graph kernel which utilizes spectral information, and is used for unlabeled datasets. Although the graph feature is well-motivated by uniqueness considerations, the features employed lack interpretability, and do not distinguish between local and global information.

Additional works that compute graph features using spectral information include NetLSD \cite{netlsd} and its approximation SLAQ \cite{slaq}. NetLSD computes a graph feature by using the heat trace signature of a graph. Since the heat kernel is a continuous analogue to the random walk, the trace is analogous to return probabilities. NetLSD is therefore conceptually similar to the techniques considered in this paper. The method is linear in $|E|$ and scales to millions of nodes. The main advantage of this work is that it captures multi-scale information about graphs. Though the authors apply NetLSD to graph classification, they use 1-nearest neighbor (1-NN) classification instead of the standard SVM approach in graph kernels literature. 


\subsubsection{Tottering}

Both RetGK $\cite{retp}$ and SAGE $\cite{sagegk}$ face an issue known as \textit{tottering} \cite{totter}: in a random walk, the local edges of a graph near the starting point are traversed back and forth over and over again, therefore over-representing local features as opposed to global features discovered by longer walks --- this is especially true when the length of the walk is short in comparison to the diameter or size of the graph and poses a bigger problem in large graphs. For random return walks, as opposed to generic random walks, the issue of tottering is further amplified, because the walk must begin and end in the same region of the graph. Although return walks effectively capture the local structural role of a node, they discount global information.

\subsection{Contributions}

\begin{itemize}
     \item We propose a novel global density of states (DOS) graph kernel for comparing global return walk probabilities using Chebyshev moments. We take advantage of Jackson damping to reduce Gibbs phenomena and motif filtering to hasten convergence of the Chebyshev moments when applying the KPM algorithm in computing feature vectors.

    \item We propose a unified density-of-states graph classification framework for (node-wise and global) return-probability-based graph kernels to ameliorate the issue of $\textit{tottering}$ in large graphs by combining local and global density of states (DOS+LDOS). This combination is natural because of the shared underlying spectral information. Our algorithms scale like $O(|E|)$, an improvement over $\text{RetGK}_{\text{I}}$ \cite{retp}, which scales like $O(N^3)$. Moreover, we show that $\text{LDOS}$ and $\text{RetGK}_{\text{I}}$ \cite{retp} are the same up to a change of basis.
   
\end{itemize}

\section{Background}
\subsection{Graph Conventions}
We consider undirected weighted graphs $G = (V, E)$ with vertex set $V = \{v_1, ..., v_n\}$ and edge set $E \subseteq V\times V$. We take the entry $(i, j)$ of the weighted adjacency matrix $A\in \mathbb{R}^{n\times n}$ to be $0$ if $(v_i, v_j)\notin E$ and $w_{ij}$, the weight of edge $(v_i, v_j)$, if the edge belongs to $E$.

\subsection{Normalized Adjacency Matrix} To guarantee that the spectral density function has bounded support within $[-1, 1]$, we consider the normalized adjacency matrix $\tilde{A} = D^{-1/2}AD^{-1/2}$, whose eigenvalues satisfy $-1\le \lambda_1\le ...\le \lambda_n \le 1$. This matrix encodes useful eigenvalue information and shares the same eigenvalues as the random walk matrix $P = D^{-1}A$. Both play a fundamental role in spectral graph algorithms. 

\subsection{Graph Kernels} Kernel methods have long been used for regression and classification tasks. The most frequently used kernel method is the support vector machine (SVM). A valid kernel $k(\bm{x}, \bm{x})$ which can be identified with a reproducing kernel Hilbert Space (RKHS) satisfies the properties of symmetry and positive definiteness. A graph kernel $K(G, H)$ is a kernel defined between graphs $G, H$, which, roughly speaking, measures their degree of similarity. A bevy of graph kernels have been proposed, including random walk kernels and numerous variations thereof (\cite{graphkernels}, \cite{totter}).

\subsection{Kernel Polynomial Method (KPM)}

To compute graph kernels based on the local and global spectral density of a matrix, we must first compute their feature vectors. The literature contains many approaches for estimating the spectral density of a large symmetric matrix. One approach hinges on running the Lanczos algorithm, to obtain approximate Ritz values and Ritz vectors and an approximate decomposition of $A$ of the form $AV_M = V_MT_M+f_Me_M^T$.  Another uses the kernel polynomial method, which we elaborate on below. For a complete review of these approaches and others, we refer to $\cite{siamreview}$.

The KPM method expands the spectral density of a matrix $A$ in terms of the dual basis of a family of orthogonal polynomials satisfying a three-term recurrence \cite{siamreview}. A commonly used family is the Chebyshev polynomials, which are given by the recurrence $T_0(x)=1, T_1(x)=x, T_{m+1}(x) = 2xT_m(x)-T_{m-1}(x)$. Since the Chebyshev polynomials are orthogonal with respect to $w(x) = 2/[(1+\delta_{0n})\pi\sqrt{1-x^2}]$, the dual basis is given by $T^* = w(x)T(x)$, and
\begin{equation}
    \mu(\lambda) = \sum_{m=1}^\infty d_mT_m^*(x)\,dx
\end{equation}
The $m$th Chebyshev moment $d_m$ is given by
\begin{equation}
d_m = \int_{-1}^1 T_m(\lambda)\mu(\lambda)\,d\lambda = \text{Tr}(T_m(A))
\end{equation}
Likewise, the local Chebyshev moments are
\label{eqn:truncate}
\begin{equation}
    d_{mk} = \int_{-1}^1 T_m(\lambda) \mu_k(\lambda)\,d\lambda = T_m(A)_{kk}
\end{equation}
Hutchinson and Bekas proposed a stochastic estimator for the diagonal of a symmetric matrix $A\in\mathbb{R}^{n\times n}$ via taking a Hadamard product:
\[\mathbb{E}[z\circ A z] = \text{diag}(A)\]
An unbiased estimator for the diagonal is given by
\[\text{diag}(A) \approx  \frac{1}{N_z} \sum_{j=1}^{N_z} Z_j \circ A Z_j\]
where $Z_j$ are independent probe vectors. The quantities $T_m(A)Z_j$ can be computed using the three-term recurrence for Chebyshev polynomials. Therefore each moment can be computed in $O(|E|N_z)$ time. 

\subsection{Jackson Damping and Motif Filtering}
A technique developed in $\cite{NDS}, \cite{siamreview}$ to accelerate the computation of the Chebyshev moments for the KPM method is motif filtering. Motif filtering preempts slow convergence caused by spikes in the spectrum of the spectral density of states by filtering them out using a hashing process, thereby reducing the problem to that of approximating smooth portions of the spectrum. We employ both techniques in our graph feature computations.

Jackson damping factors can be used to control the Gibbs oscillation caused by truncating the number of Chebyshev moments in equation \ref{eqn:truncate}. The Jackson factors are given by
\[ g_M = \dfrac{(M+1-k)\cos(k\alpha_M)}{(M+2)}+\dfrac{\sin((k+1)\alpha_M)}{(M+2)\sin (\alpha_M)}\]
where $\alpha_M = \frac{\pi}{M+2}$ and $M$ is the number of moments.



\section{Methods}
\subsection{DOS Graph Kernel}
Common graph motifs, or recurring substructures, are reflected in the spectrum of the graph Laplacian or adjacency matrix in the form of repeated eigenvalues \cite{NDS}. This motivates a graph kernel which compares frequently occuring graph motifs between graphs. We do this indirectly by comparing the feature vector of Chebyshev moments $d_m$ of the normalized adjacency matrix $\tilde{A}$. The explicit DOS Euclidean embedding vector is given by
\begin{equation}
    \mu_{DOS}(G) = [\textbf{R}_G^1, \textbf{R}_G^2,\hdots, \textbf{R}_G^N]^T
\end{equation}
where  $\textbf{R}_G^S := \int_{-1}^1 T_S(\lambda) \,d\mu(\lambda)$ and $\mu(\lambda)$ is the generalized function as defined in equation $\ref{eqn:dos}$ and $T_S$ is the $S$th Chebyshev polynomial. Accordingly, we define the DOS graph kernel between graphs $G$ and $H$ as
\begin{equation}
K(G, H) = \exp(-\gamma\|\mu_{DOS}(G)-\mu_{DOS}(H)\|_2^p)
\end{equation}
The choices of $p=1$ and $p=2$ correspond to the Laplacian and RBF kernels, respectively. 

For the special case of the normalized adjacency matrix, the $m$th moment $d_m$ has the additional interpretation as a random walk return probability, namely the probability that a random walk starting at a node chosen uniformly at random returns to itself after $m$ steps along the edges of the graph.

\subsection{LDOS Graph Kernel}
We begin by recollecting a fact that relates powers of the normalized adjacency matrix $\tilde{A}$ to return probabilities on graphs. Then we use the density of states framework to compute the probabilities.
\begin{theorem}[Return Probabilities]
If $\tilde{A}$ is the normalized adjacency matrix of graph $G$, then the probability that a random walk starting from node $i$ returns to node $i$ after $k$ steps is given by $[\tilde{A}^k]_{ii}$.
\end{theorem}
This fact is easily proved by an inductive argument. We demonstrate how the LDOS encodes exactly this set of return probabilities; all that is needed to recover this information is a change-of-basis from the Chebyshev basis to the monomial basis.

By definition the pointwise density of states (PDOS) $\mu_j(\lambda)$ is defined as 
\[\mu(\lambda, e_j):= \sum_{i=1}^n |e_j^Tq_i|^2\delta(\lambda-\lambda_i)\]
where $q_i$ is the $i$th eigenvector of $\tilde{A}$ corresponding to $\lambda_i$. 
\begin{proposition}[Return Probabilities V2]
The probablility that a random walk starting at node $v_i$ returns after $k$ steps is $\int_{-1}^1 \lambda^k \,d\mu_i(\lambda)$, where $\mu_i(\lambda):=\mu(\lambda, e_i)$.
\label{prop:return}
\end{proposition}
\begin{proof}
Assume $\tilde{A} = QDQ^T$, where $Q$ is an orthonormal matrix of eigenvectors and $D$ is a diagonal matrix of eigenvalues. Observe that
\begin{align*}
\begin{split}
\label{eqn:PDOS}
    \int_{-1}^1 \lambda^k \,d\mu_j(\lambda) &= \sum_{i=1}^n \lambda_i^k |e_j^Tq_i|^2 = e_j^TQD^kQ^Te_j\\
    &= e_j^T(QDQ^T)^ke_j = e_j^T \tilde{A}^k e_j = [\tilde{A}^k]_{jj}
\end{split}
\end{align*}
Hence the integral $\int_{-1}^1 \lambda^k \,d\mu_j(\lambda)$ is equal to the probability that a random walk of length $k$ on $G$ starting at node $i$ returns to the initial point.
\end{proof}

 From proposition \ref{prop:return}, we have that
\[\int_{-1}^1 T_k(\lambda)\,d\mu_j(\lambda) = T_k(\tilde{A})_{jj}\]
where $T_k(\lambda)$ is the $k$th Chebyshev polynomial.
We use the same stochastic diagonal estimation approach to accurately compute the matrix of moments $C \in \mathbb{R}^{K \times n}$, where $C_{ij}$ is the $i$th Chebyshev moment for the $j$th node, $T_{i}(\tilde{A})_{jj}$.

The return probabilities feature vector (RPF), or vector of polynomial moments, is then obtained from the $T_k(\tilde{A})_{ij}$ by solving the lower trangular linear system using forward-substitution. 
\label{eqn:system}
\begin{equation}
    \begin{bmatrix} T_{00} & 0 & \hdots & 0 \\ T_{10} & T_{11} & \hdots & 0 \\ \vdots & \hdots & \ddots & \vdots \\ T_{k0} & T_{k1} & \hdots & T_{kk} \end{bmatrix} \begin{bmatrix}\int_{-1}^1 \lambda^0 d\mu_i\\ \int_{-1}^1 \lambda^1 d\mu_i\\ \vdots \\ \int_{-1}^1 \lambda^kd\mu_i \end{bmatrix} = \begin{bmatrix} \int_{-1}^1 T_0(\lambda) d\mu_i \\ \int_{-1}^1 T_1(\lambda) d\mu_i\\\vdots \\ \int_{-1}^1 T_k(\lambda)d\mu_i \end{bmatrix} 
\end{equation}
Forward substitution is a numerically backwards stable algorithm. In practice, one commonly sets $k=50$ \cite{retp}.  

\subsubsection{Error Analysis} Two sources of error exist in computing the LDOS moments in the polynomial basis using equation \ref{eqn:DOSFormula}. The first is the error in computing the approximate Chebyshev moments $\{\int_{-1}^1 T_j(\lambda)\,d\mu_i\}_{j=0}^k$. The second is floating point error from solving the linear system in equation \ref{eqn:system}. It turns out that the error in the approximate Chebyshev moments is structured. There is a degree of internal consistency in the moment approximations because the individual moments are not approximated independently. Rather the approximate moments are exact moments for a distribution which approximates the density of states, namely
\begin{equation}
    \mu(\lambda, u) = \sum_{i=1}^N |u^Tq_i|^2\delta(\lambda-\lambda_i)
    \label{eqn:DOSFormula}
\end{equation}
where $u$ is usually taken to be a Gaussian random vector. That this quantity is an unbiased estimator of the DOS is implied by the following theorem $\cite{siamreview}$.

\begin{theorem}[Stochastic Sampling]
Let $A$ be a symmetric matrix in $\mathbb{R}^n$ with eigenvalue decomposition $A = \sum_{j=1}^n\lambda_j u_ju_j^T$ satisfying $\langle u_i, u_j \rangle = \delta_{ij}$ for $i, j = 1, ..., n$. If $v$ is a vector in $\mathbb{R}^n$, then it can be represented as $v = \sum_{j=1}^n \beta_j u_j$. If each component of $v$ is drawn from a normal distribution with zero mean and unit standard deviation, i.e.
\[\mathbb{E}[v] = 0, \mathbb{E}[vv^T] = I\]
then
\[\mathbb{E}[\beta_i\beta_j] = \delta_{ij}, \quad i, j = 1, 2, ..., n\]
\end{theorem}
The second source of error is related to the solve involving the matrix of Chebyshev matrix, which is poorly scaled. As such, we can use the Skeel condition number to more appropriately quantify the conditioning of the matrix than the $2$-norm condition number. The Skeel condition number has the property of invariance under row scaling, and is formally defined as
\[\text{Cond}(T) = \||T||T^{-1}|\|_\infty \]
where $|T|$ is the entry-wise absolute value operation. The Skeel condition number of $T_{50}$ is smaller than the $2$-norm condition number by a factor of $10^4$.
\subsubsection{Computation of LDOS Kernel}
We summarize these steps in the form of an algorithm for computing RPF via LDOS.
\label{algo:ldos}
\begin{algorithm}[H]
\caption{LDOS Return Probabilities Algorithm}
\begin{algorithmic}[1]\label{Algo:pdos}
\STATE \textbf{Inputs}: Normalized adjacency matrix $\tilde{A}$ for graph $G$ with eigenvalues in $[-1, 1]$, orthonormal basis $P$ for motif hashing \\ \STATE \textbf{Outputs:} Feature vector of return probabilities $v \in\mathbb{R}^{(k+1)n}$
\STATE Sample probe vectors $\{Z_j\}_{j=1}^{N_Z}$
\STATE Compute $\{Z_r\}_{r=1}^{N_Z} = \{PP^TZ_j\}_{j=1}^{N_Z}$, $P\in \mathbb{R}^{N\times r}$
\STATE $C \in\mathbb{R}^{(k+1)\times n} \leftarrow \frac{1}{N_Z}\sum_{r=1}^{N_Z} Z_r\circ AZ_r$
\STATE $C \leftarrow  J_MC$ (Jackson Damping)
\STATE Compute $V = T^{-1}C \in \mathbb{R}^{(k+1)\times n}$, where $T\in\mathbb{R}^{(k+1)\times (k+1)}$ contains Chebyshev coefficients $(3.9)$\nolinebreak[2]
\STATE Flatten $V$ to obtain $v\in \mathbb{R}^{n(k+1)}$
\STATE \textbf{Return $V$}
\end{algorithmic}
\end{algorithm}

As in $\cite{retp}$, given a set of $k$-step return probabilities feature vector (RPF) $\vec{\bm{p}}_i$ for each node of $G$, we assume they form an empirical distributions and embed them in a reproducing kernel Hilbert Space (RKHS) using kernel mean embedding. We use maximum mean discrepancy (MMD) \cite{MMD} to obtain similarity measures between empirical distributions.

Following \cite{retp}, we take the graph kernel to be
\begin{align*}
\begin{split}
    K_{LDOS}(G, H) &= \exp(-\gamma\|\mu(G) - \mu(H)\|_{\mathcal{H}}^p)
    \\&= \exp(-\gamma \text{MMD}^p(\mu_{G_1}, \mu_{G_2}))
\end{split}
\end{align*}
where 
\begin{equation*}
\begin{split}
\text{MMD}(\mu_G, \mu_H) &= \frac{1}{n_G^2} \textbf{1}_{n_G}^T K_{GG} \textbf{1}_{n_G} +\frac{1}{n_H^2}\textbf{1}_{n_H}^T K_{HH} \textbf{1}_{n_H} \\& - \frac{2}{n_Gn_H} \textbf{1}_{n_G}^TK_{GH}\textbf{1}_{n_H}
\end{split}
\end{equation*}
and $\mu_G, \mu_H$ are taken to be either vectors of LDOS moments or node-wise return probabilities for $G$ and $H$, e.g. the output of algrithm \ref{algo:ldos}.

\subsection{Unified DOS and LDOS Graph Kernel}
Motivated by the non-overlapping advantages of the DOS and LDOS kernels, we consider techniques from multiple-kernel learning to combine the two kernels. As used in \cite{cocabo}, a natural choice is the composite kernel given by the linear combination of sum and products
\begin{equation}
    K(G, H) = w_1K_{\text{DOS}} \cdot K_{\text{LDOS}} + \frac{w_2}{2}(K_{\text{DOS}}+ K_{\text{LDOS}})
    \label{eqn:combokernel}
\end{equation}
where $w_1+w_2 = 1$. In other words, we take a linear combination of sum and product kernels to create a more expressive kernel that is sensitive to both global and local information. 

From a geometric point of view, the DOS is able to detect global graph motifs: frequently occuring substructures in the graph, formed by nodes and edges in a specific way. The most common type of graph motif is the dangling vertex. Specific graph motifs give rise to certain eigenvalues, which in turn impact the shape of the spectral density, or DOS. While the DOS captures the presence of graph motifs, the LDOS effectively captures the local node structural roles: what is the function of a given node in the context of its local neighborhood, as reflected by random walks?

From a probabilistic point of view, the DOS encapsulates the probability that a random walk returns after $S$ steps, given that the starting point is chosen uniformly at random among the nodes of the graph $G$. The LDOS, on the other hand, captures the probability that a random walk returns after $S$ steps given a fixed starting node. 

This composite graph kernel is most effective when either one of DOS or LDOS is inadequate for summarizing a graph alone, as is often the case when a graph is very small or very large.  

The sum of positive definite matrices is positive definite; likewise, the element-wise product is also positive definite \cite{bhatia}. We conclude that the kernel defined in equation $\ref{eqn:combokernel}$ is positive definite. 

The technique of combining graph kernels does not introduce external, or new information to $\text{RetGK}_{\text{I}}$ when the number of moments used is the same. Instead it utilizes the same information in a more efficient manner. We are not combining two distinct kernels, as the underlying information used is the same. In fact, we recover the global DOS by summing over the LDOS moments.

\begin{algorithm}[H]
\caption{LDOS + DOS Graph Feature}
\begin{algorithmic}[1]\label{Algo:pdos_combo}
    \STATE \textbf{Inputs}: Normalized adjacency matrix $\tilde{A}$ for graph $G$ with eigenvalues in $[-1, 1]$ \\ \textbf{Outputs:} Feature vector of node-wise return probabilities $V \in\mathbb{R}^{(k+1)n}$  and global return probabilities $W \in \mathbb{R}^{k+1}$
    \STATE $C \in\mathbb{R}^{(k+1)\times n} \leftarrow \texttt{HutchinsonEstimator}(\tilde{A})$
    \STATE Compute $V = T^{-1}C \in \mathbb{R}^{(k+1)\times n}$, where $T\in\mathbb{R}^{(k+1)\times (k+1)}$ contains Chebyshev coefficients
    \STATE $W \leftarrow \sum_{i=1}^n v_i$
    \STATE \textbf{Return $V, W$}
\end{algorithmic}
\end{algorithm}
The time complexity of the algorithm is $O(|E|N_ZS)$, where $S=k+1$ is the total number of moments to be computed for graph.

\section{Experiments}
We experiment on a wide range of network-based graph datasets, with descriptions provided below. We denote the global density of states graph kernel simply by DOS, and the local density of states graph kernel simply by LDOS. We emphasize that LDOS is a scalable, efficient approximation to the return probability kernel $\text{RetGK}_{\text{I}} \cite{retp}$. In our experiments, we rerun $\text{RetGK}_{\text{I}}$ for all datasets, rather than citing values from literature. We combine DOS with both and LDOS and $\text{RetGK}_{\text{I}}$ to correct for the effects of tottering. For details of experimental setup and hyperparameter selection, see \S 4.2.

\begin{table*}
    \centering
    \small
    \caption{Classification accuracies (\%) and standard deviations. $\text{RetGK}_{\text{I}}$ values were re-computed, except for COLLAB. DOS, LDOS, LDOS+DOS, and $\text{RetGK}_{\text{I}}$+DOS values were computed. Remaining values were cited.} 
    \begin{tabular}{|l|cccccc|}\hline
    {}& \multicolumn{4}{c|}{Medium/Large Graphs} & \multicolumn{2}{c|}{Small Graphs} \\ \hline
{} &     REDDIT &     REDDIT-5K &  THREADS   & COLLAB & IMDB-MULTI & IMDB-BINARY\\
nodes   & 430  &   508.52 & 91.3 & 74.49 & 13 & 19.77 \\
edges   & 498  &  595   & 104.4 & 2458 &  65.94 & 96.53\\
graphs  & 2000 & 4999   & 642 & 5000 &  1500 & 1000\\
classes & 2    & 5      & 2 & 3 & 3 & 2 \\
\hline
FGSD \cite{fgsd}  & 86.50  & 47.76 & -- & 80.02 & \textbf{52.41}  & \textbf{73.62} \\
DGK \cite{retp} & 78.0(0.4) &  41.3(0.2) & -- & 73.1(0.3) & 44.6 (0.5) &  67.0(0.6)  \\
WL \cite{retp}  & 68.2(0.2) & 51.2(0.3) & -- & 74.8 (0.2)  & 49.8 (0.5) & 70.8 (0.5) \\
$\text{RetGK}_{\text{I}}$ & 86.95(0.36) & 55.67(0.19) & 77.78 (0.18) & 81.0 (0.3) & 47.12 (0.84) & 71.92 (0.75) \\
 \hline
DOS & 88.80(0.29) & 52.83(0.21) & 78.90 (0.56) & 80.78 (0.24) & 49.42 (0.45) & 72.77 (0.86) \\
LDOS & 91.19(0.26) & 55.67(0.17) & 71.03 (1.2)  & 75.29 (0.44) & 48.33 (0.74) & 71.02 (0.37)\\
\hline
LDOS+DOS & \textbf{92.48}(0.24) & \textbf{56.84}(0.20)  & 78.53 (0.43)  & \textbf{81.12} (0.18) & 49.00 (0.30) & 72.60 (0.57)   \\
$\text{RetGK}_{\text{I}}$ + $\text{DOS}$ & 91.2 (0.28) & {56.75} (0.18) & \textbf{80.11} (0.45) & -- & 50.96 (0.35) &72.25 (0.49) \\ 
 \hline
    \end{tabular}
\label{table:baselines}
\end{table*}
\subsection{Experimental Datasets}
\begin{itemize}
    \item \textbf{REDDIT} \cite{KKMMN2016} is a social network dataset. Each graph represents a discussion thread on Reddit, with nodes corresponding to users and edges between users if one has replied to the comment of another. The graphs are classified into two categories based on thread type:\vspace{-\baselineskip}\vspace{1mm}
    \begin{itemize}
        \item Threads from Q\&A subreddits (ie. r/AskReddit)
        \item Threads from online discussion type subreddits (ie. r/atheism).
    \end{itemize}
    \item \textbf{REDDIT-5K} \cite{KKMMN2016} is similar to REDDIT, except with 5 subreddits labeled independently: worldnews, videos, AdviceAnimals, aww, and mildlyinteresting.
    \item  \textbf{COLLAB} \cite{KKMMN2016} is a physics research collaboration network graph dataset. Each graph corresponds to the research network for a given researcher, nodes correspond to researchers, and edges between researchers indicate they've collaborated with. Each graph is classified based on the physics subfield the researcher belongs to.
    \item \textbf{IMDB-BINARY} \cite{KKMMN2016} Is a set of ego-networks of actors. In each graph, nodes correspond to actors/actresses, and edges between actors indicate they starred in the same movie. The graphs are classified into two categories based on the genre of movie: Action and Romance.
    \item \textbf{IMDB-MULTI} \cite{KKMMN2016} Is the same as IMDB-BINARY, except using three genres: Comedy, Romance and Sci-Fi.
    \item \textbf{REDDIT-THREADS} \cite{karate} 
    Discussion and non-discussion based threads from Reddit. We filtered the dataset to only include graphs with $>87$ nodes and balanced the number of instances in each class.
\end{itemize}

\begin{figure}[th]
\centering
    \includegraphics[width=.53\textwidth]{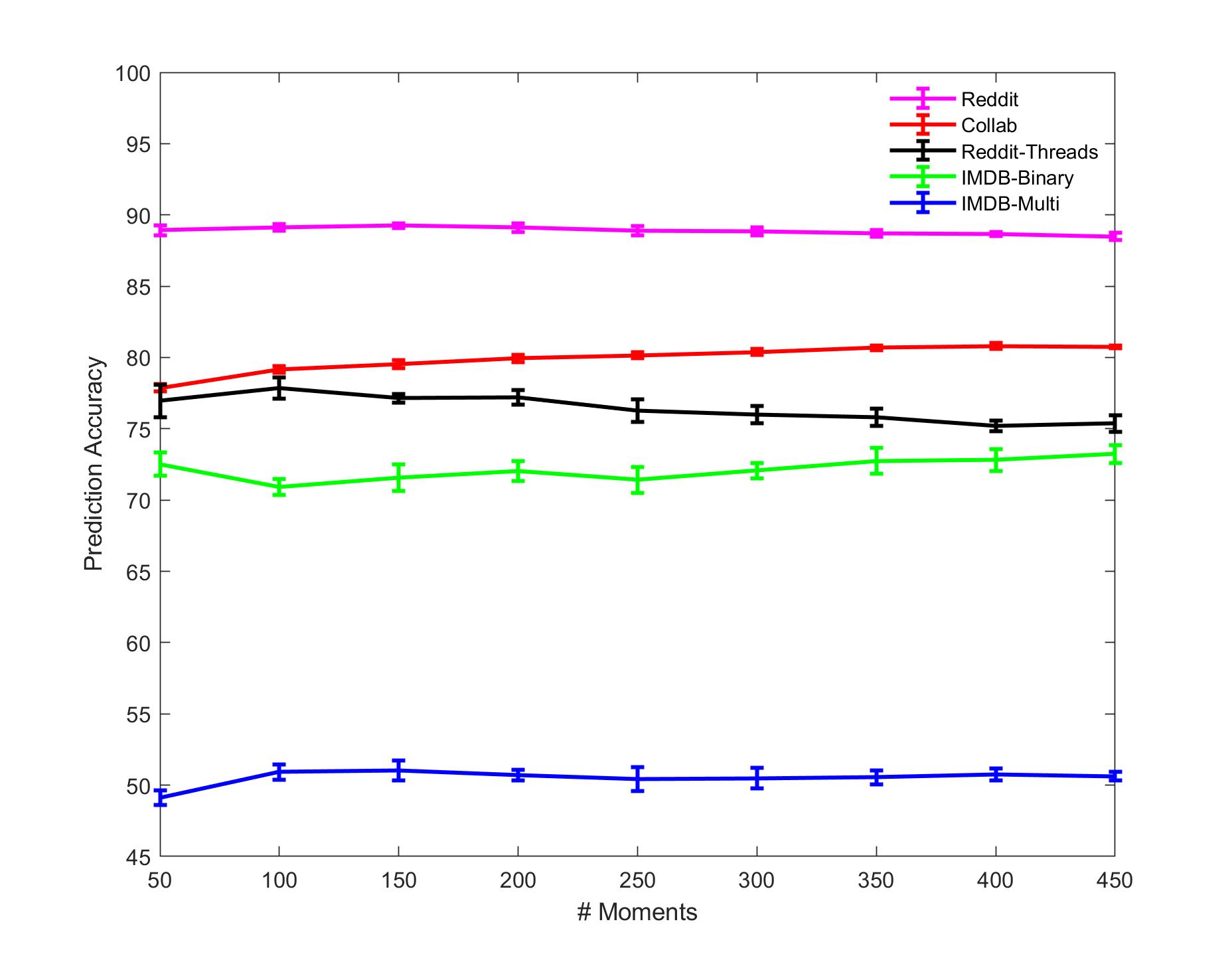}
    \caption{DOS prediction accuracy and standard deviation versus number of Chebyshev moments.}
\label{fig:Benchmarks1}
\end{figure}

\begin{figure}[th]
\centering
    \includegraphics[width=.53\textwidth]{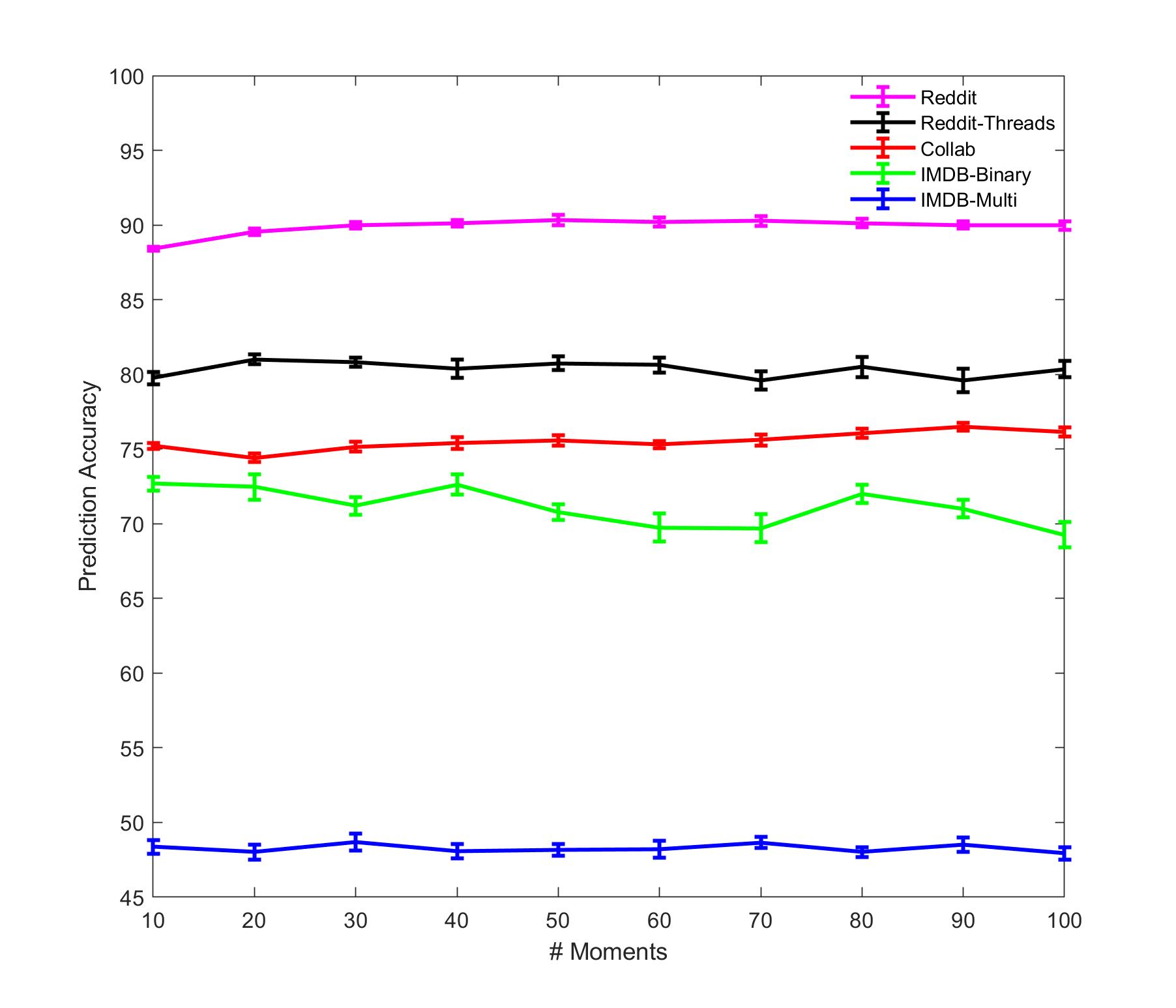}
    \caption{LDOS prediction accuracy and standard deviation versus number of Chebyshev moments.}
\label{fig:Benchmarks2}
\end{figure}

\subsection{Hyperparameter Selection}
Following $\cite{graphkernelsurvey}$, we select the SVM parameter $C$ from $\{10^{-3}, 10^{-2}, ..., 10^3\}$ using cross-validation on the training set. We optimize $C$ using the training set only. We perform the outer $10$-fold cross-validation loop $10$ times with random splits and report mean accuracy as well as standard deviation. For generating the DOS graph kernel (3.1), we used $50-400$ Chebyshev moments and $2000$ random probe vectors. For generating the LDOS graph kernel, we used $50$ Chebyshev moments and $2000$ probe vectors. We set the number of moments to be $50$ because that is the number of steps taken in the random walk in $\text{RetGK}_{\text{I}}$. As in \cite{retp}, we select the $p$-value in the from $\{1, 2\}$. In all cases we choose $\gamma$ to be the median of entries in the kernel matrix. In the composite kernel, we always take the weights to be $w_1 = w_2 = \frac{1}{2}$ for simplicity. We found that prediction accuracy was insensitive to these weights. In motif filtering, we chose to filter out eigenvalues $\{-1/4, -1/3, -1/2, 0\}$, as done in \cite{NDS}.

We used the $\text{C-SVM}$ implementation of $\text{LIBSVM}$ (Chang, et al) to conduct all experiments. We used code from $\cite{NDS}$ to compute DOS and LDOS feature vectors. Return probability kernels were computed using NVIDIA Titan V GPU with 5120 1.455GHz cores, and DOS kernels were computed using an Intel Xeon with 12 cores. 

\subsection{Discussion}
Even though $\text{LDOS}$ is an $O(|E|)$ approximation to $\text{RetGK}_{\text{I}}$, which is scales like $O(N^3)$, we observed competitive performance between the two, indicating that the accuracy of our approximation was sufficient for classification tasks. Moreover, when we combined the LDOS kernel with the DOS kernel, we observed that it uniformly boosted the classification accuracy of LDOS across medium/large graphs with greater than $50$ nodes on average. When $\text{RetGK}_{\text{I}}$ was combined with DOS, we saw a similar increase in accuracy. We attribute this to the incorporation of global return probability information. 

The new DOS kernel has reasonable performance across a range of datasets, however, it had the highest performance on datasets with small graphs. On small graphs with $<20$ nodes on average, the RPF feature of length $50$ was less capable of telling apart graphs than DOS graph kernel, likely because the graphs had limited size. For small graphs, the DOS+LDOS combination had less visible effect, possibly because the distinction between local and global information is blurred. We recommend using FGSD in the regime of small graphs, and recommend LDOS+DOS for medium and large graphs.


Through our experiments, we illustrated the benefits of combining the DOS and LDOS kernels using techniques from multiple kernel learning. In all large graph datasets, the resulting combination yielded higher accuracy than any single kernel alone. These results suggest that our approach of reusing spectral information (node-wise Chebyshev moments $d_{mk}$) in two different ways (yielding the DOS and LDOS graph features, respectively), has the potential to enhance classifier performance across large graph datasets --- and at minimal extra cost over computing the LDOS moments, because the DOS moments are obtained by summing over the LDOS moments.


\begin{figure}[th]
\centering
    \includegraphics[width=.5\textwidth]{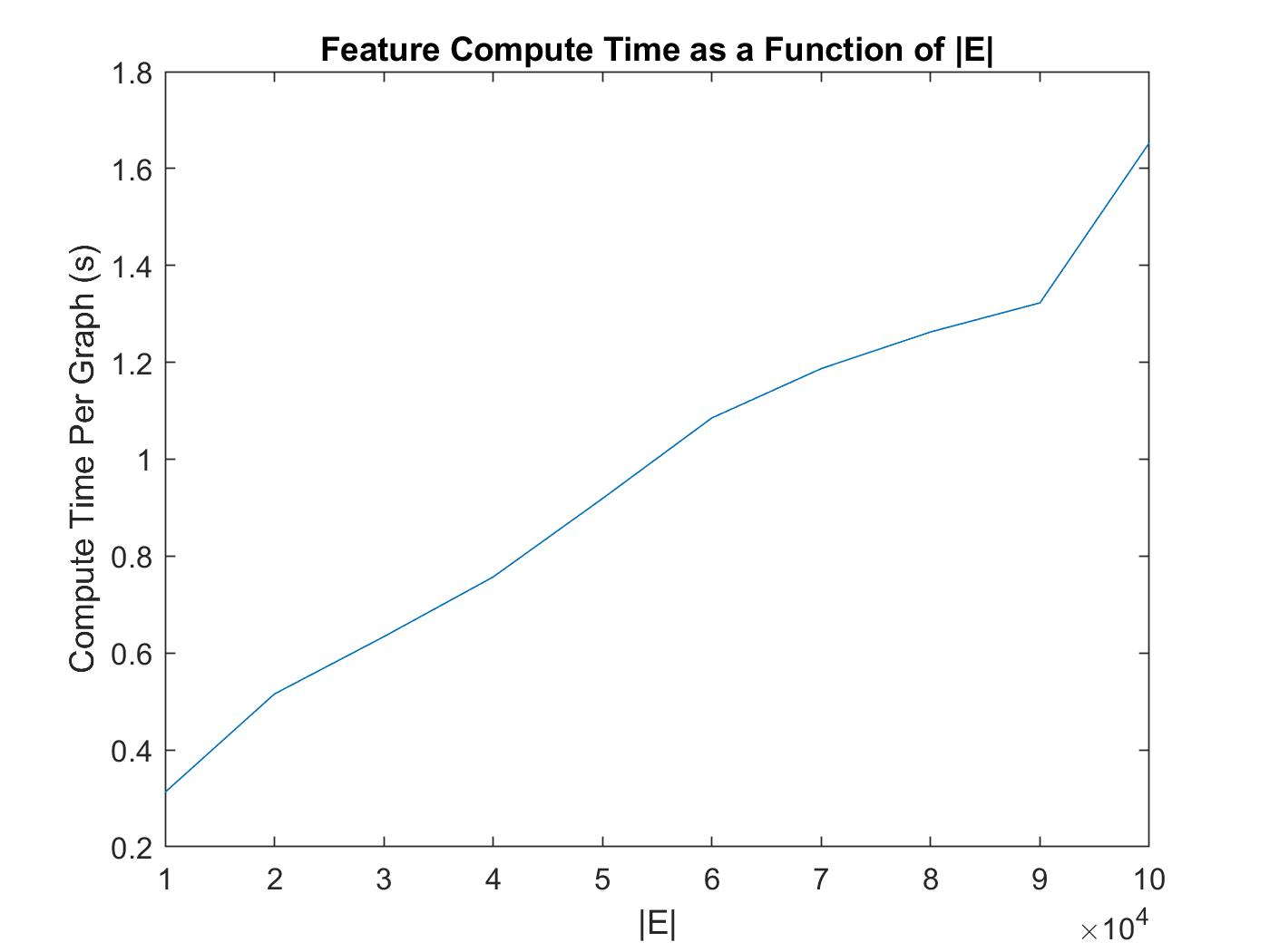}
    \caption{The LDOS Feature generation run time scales linearly with the number of edges $|E|$ in $G$.}
\label{fig:LDOS_Scaling}
\end{figure}

\section{Conclusion}

In this paper, we introduced a unified local and global DOS graph kernel to amend the problem of tottering in random walks, and to enable return probability-based graph kernels to scale to large graphs with tens of thousands to millions of nodes. Tottering is a problem which reduces the expressive power of random walks by placing overemphasis on local structure and discounting global, or faraway structures. By explicitly taking global structure into account, we were able to offset the effect of tottering and improve the performance of return probability graph kernels across a range of large graph datasets where a random walk of length $50$ is incapable of weighing local and global information adequately. We outperform state-of-the-art graph kernels on all of the large unlabeled graph datasets considered. 

Our methods are inherently scalable to graphs with millions of nodes and hundreds of millions of edges \cite{NDS}, unlike $\text{RetGK}_{\text{I}}$ and $\text{SAGE}$. By relying on matrix-vector multiplications (MVMs), our algorithm scales like $O(|E|)$ and naturally exploits sparsity in the graph adjacency matrix. At the same time, we showed empirically that our method remains competitive with state-of-the-art graph kernels for small, medium, and large sized graphs. The DOS/LDOS graph kernel is readily extended to graphs with both node and edge attributes, since the algorithm for return probabilities generates the same features as $\text{SAGE}$ \cite{sagegk}, which successfully handles both node and edge features by taking the adjoint.   

In the future we would like to investigate whether our DOS and LDOS-based graph feature vectors can used in the pre-training phase of GNNs, similar to how $\cite{fewshot}$ exploits the Wasserstein barycenter to create super-graphs. We would like to investiate replacing the entry-wise geometric mean in the composite kernel by the matrix geometric mean $A \# B$. Lastly, we would like to examine how much the global DOS can be used to enhance classification accuracy on node-attributed and edge-attributed datasets.

\section{Acknowledgements}
We thank Yuanzhe Xi and Anil Damle for their helpful comments and suggestions. This work was supported in part by the Hummer-Tuttle gift to Professor Al Barr through the Caltech Division of Engineering and Applied Science.


\newpage
\begin{thebibliography}{99}

\bibitem{netlsd}
Anton Tsitsulin et al. {\em Netlsd: hearing the shape of a graph}. Proceedings of the 24th ACM SIGKDD International Conference on Knowledge Discovery \& Data Mining. 2018.

\bibitem{slaq}
Anton Tsitsulin, Marina Munkhoeva, and Bryan Perozzi. {\em Just slaq when you approximate: Accurate spectral distances for web-scale graphs}. Proceedings of The Web Conference 2020. 

\bibitem{nbcqanon}
Ari Sen and Brandy Zadrozny. {\em QAnon groups have millions of members on Facebook, documents show}. NBC News. 10 Aug, 2020. Web. 12 Oct, 2020.

\bibitem{MMD}
Arthur Gretton et al. {\em A kernel two-sample test}. Journal of Machine Learning Research, 13(Mar):723–773, 2012.

\bibitem{karate}
Benedek Rozemberczki, Oliver Kiss, and Rik Sarkar. {\em An API Oriented Open-source Python Framework for Unsupervised Learning on Graphs}. International Conference on Information and Knowledge Management, 2020.

\bibitem{cocabo} 
Binxin Ru, Ahsan S. Alvi, Vu Nguyen, Michael A. Osborne, and Stephen J Roberts.
{\em Bayesian Optimisation over Multiple Continuous and Categorical Inputs}. 3rd Workshop on Meta-Learning at NeurIPS 2019, Vancouver, Canada.

\bibitem{tudata} Christopher Morris, Nils M. Kriege, Franka Bause, Kristian Kersting, Petra Mutzel and Marion Neumann. {\em TUDataset: A collection of benchmark datasets for learning with graphs}

\bibitem{Golub}
Gene Golub and Charles Van Loan. {\em Matrix Computations}. Johns Hopkins University Press, Baltimore MD, 4 edition, 2012.

\bibitem{gromov}
Hongteng Xu, Dixin Luo, Hongyuan Zha, and Lawrence Carin
{\em Gromov-Wasserstein Learning for Graph Matching and Node Embedding.} Proceedings of the 36th International Conference on Machine Learning, PMLR 97:6932-6941, 2019.


\bibitem{fewshot}
Jatin Chauhan, Deepak Nathani, and Manohar Kaul. {\em Few-Shot Learning on Graphs via Super-Classes based on Graph Spectral Measures}. ICLR, 2020.

\bibitem{graphkernelsurvey}
Johansson Kriege and Morris. {\em A survey on graph kernels}. Applied Network Science, 5, 2020.

\bibitem{gin}
Keyulu Xu, Weihua Hu, Jure Leskovec, Stefanie Jegelka {\em How Powerful are Graph Isomorphism Networks?} ICLR, 2019.

\bibitem{graphkernels}
Kondor Vishwanathan et al. {\em Graph kernels}. Journal of Machine Learning Research, 2010.


\bibitem{KKMMN2016} Kristian Kersting et al. {\em Benchmark Data Sets for Graph Kernels}.


\bibitem{NDS}
Kun Dong, Austin R. Benson, and David Bindel. {\em Network density of states}. Proceedings of the 25th ACM SIGKDD International Conference on Knowledge Discovery \& Data Mining, July 2019, Pages 1152–1161


\bibitem{sagegk}
Lingfei Wu, Zhen Zhang, Arye Nehorai, Liang Zhao, and Fangli Xu. {\em SAGE: Scalable Attributed Graph Embeddings For Graph Classification}. ICLR, 2019.


\bibitem{totter}
Mahe, P., Ueda, N., Akutsu, T., Perret, J.-L., and Vert, J.-P. {\em Extensions of marginalized graph kernels}. In Proceedings of the 21st International Conference on Machine Learning (ICML), 2004.

\bibitem{PropKer}
Marion Neumann et al. {\em Propagation kernels: efficient graph kernels from propagated information}. Mach Learn 102, 209–245 (2016).



\bibitem{wassleh}
Matteo Togninalli et al.  {\em Wasserstein Weisfeiler-Lehman Graph Kernels}. In Proceedings of the 33rd Conference on Neural Information Processing Systems (NeurIPS 2019)


\bibitem{aaai} Nikolentzos, G., Meladianos, P., and Vazirgianni, M.
{\em Matching Node Embeddings for Graph Similarity.} In Proceedings of the 31st AAAI Conference on Artificial Intelligence, 2017.

\bibitem{wloa}
Nils M. Kriege and Pierre-Louis Giscard and Richard C. Wilson. {\em On Valid Optimal Assignment Kernels and Applications to Graph Classification}. 30th Conference on Neural Information Processing Systems (NIPS 2016).

\bibitem{wl}
Nino Shervashidze and Pascal Schweitzer and Erik Jan van Leeuwen and Kurt Mehlhorn and Karsten M. Borgwardt. {\em Weisfeiler-Lehman Graph Kernels}. Journal of Machine Learning Research, 2011.


\bibitem{dgk}
Pinar Yanardag and S.V.N. Vishwanathan {\em Deep Graph Kernels}, In Proceedings of the 21th ACM SIGKDD International Conference on Knowledge Discovery and Data Mining (KDD 2015)



\bibitem{bhatia}
Rajendra Bhatia, {\em Matrix Analysis}. Springer, 1997.


\bibitem{fgsd}
Saurabh Verma and Zhi-Li Zhang. {\em Hunt For The Unique, Stable, Sparse And Fast Feature Learning On Graphs }. NeurIPS, 2017.


\bibitem{realgraphs}
Siddhartha Sahu, Amine Mhedhbi, Semih Salihoglu, Jimmy Lin, M. Tamer Özsu. {\em The Ubiquity of Large Graphs and Surprising Challenges
of Graph Processing}. Proceedings of the VLDB Endowment, 2017.

\bibitem{gntk}
Simon S. Du et al.
{\em Graph Neural Tangent Kernel: Fusing Graph Neural Networks with Graph Kernels}. NeurIPS, 2019.


\bibitem{fastrandom}
U Kang and Hanghang Tong and Jimeng Sun. {\em Fast Random Walk Graph Kernel}. SIAM International Conference on Data Mining, 2012.


\bibitem{ogb} 
Weihua Hu and Matthias Fey and Marinka Zitnik and Yuxiao Dong and Hongyu Ren and Bowen Liu and Michele Catasta and Jure Leskovec. {\em Open Graph Benchmark: Datasets for Machine Learning on Graphs}. CoRR abs/2005.00687


\bibitem{GNNsurvey}
Wu Z, Pan S, Chen F, et al. {\em A Comprehensive Survey on Graph Neural Networks}. IEEE Transactions on Neural Networks and Learning Systems. 2020 Mar. DOI: 10.1109/tnnls.2020.2978386.

\bibitem{collabsize}
Jure Leskovec and Andrej Krevl. {\em SNAP Datasets: Stanford Large Network Dataset Collection}. \url{http://snap.stanford.edu/data}. 2014 Jun.

\bibitem{siamreview}
Yousef Saad Lin Lin and Chao Yang. {\em Approximating spectral densities of large matrices}. SIAM Rev, 2016.

\bibitem{retp}
Zhen Zhang, Mianzhi Wang, Yijian Xiang, Yan Huang, and Arye Nehorai. {\em RetGK: Graph Kernels based on Return Probabilities of Random Walks}. NeurIPS, 2018.



\end{thebibliography}
\end{document}